\documentclass{article}

\usepackage{PRIMEarxiv}

\usepackage[utf8]{inputenc} 
\usepackage[T1]{fontenc}    
\usepackage{hyperref}       
\usepackage{url}            
\usepackage{booktabs}       
\usepackage{amsfonts}       
\usepackage{nicefrac}       
\usepackage{microtype}      
\usepackage{lipsum}
\usepackage{fancyhdr}       
\usepackage{graphicx}       
\graphicspath{{media/}}     

\usepackage{tabularx}
\usepackage{amssymb}
\usepackage{tikz}
\usetikzlibrary{patterns}
\usepackage{times}
\usepackage{soul}
\usepackage{colortbl}
\usepackage{caption}
\usepackage{amsmath}
\usepackage{amsthm}
\usepackage{algorithm}
\usepackage{algorithmic}
\usepackage[switch]{lineno}

\newtheorem{theorem}{Theorem}
\title{Pack-PTQ: Advancing Post-training Quantization of Neural Networks by Pack-wise Reconstruction
}

\author{
  Changjun Li, Runqing Jiang, Zhuo Song, Pengpeng Yu, Ye Zhang, Yulan Guo \\
  Shenzhen Campus, Sun Yat-sen University \\
  Aviation University of Air Force\\
  \texttt{\{lichj68, jiangrq3\}@mail2.sysu.edu.cn} \\
}

\begin{document}
\maketitle

\begin{abstract}
    Post-training quantization (PTQ) has evolved as a prominent solution for compressing complex models, which advocates a small calibration dataset and avoids end-to-end retraining.
    However, most existing PTQ methods employ block-wise reconstruction, which neglects cross-block dependency and exhibits a notable accuracy drop in low-bit cases. To address these limitations, this paper presents a novel PTQ method, dubbed Pack-PTQ. First, we design a Hessian-guided adaptive packing mechanism to partition blocks into non-overlapping packs, which serve as the base unit for reconstruction, thereby preserving the cross-block dependency and enabling accurate quantization parameters estimation. Second, based on the pack configuration, we propose a mixed-precision quantization approach to assign varied bit-widths to packs according to their distinct sensitivities, thereby further enhancing performance. Extensive experiments on 2D image and 3D point cloud classification tasks, using various network architectures, demonstrate the superiority of our method over the state-of-the-art PTQ methods.
\end{abstract}


\section{Introduction}

   The remarkable progress of neural networks has propelled significant breakthroughs in computer vision, with notable achievements in image classification~\cite{dosovitskiy2020image,touvron2021training}, object detection~\cite{liu2021swin,he2017mask,yang2024dual}, and semantic segmentation~\cite{strudel2021segmenter,xie2021segformer}. However, despite their impressive performance, deploying neural networks on resource-constrained edge devices (e.g., smartphones) remains a significant challenge due to their substantial computational and memory requirements. To handle this, model compression techniques have gained significant attention recently, which aim to compress and accelerate complex neural networks without compromising accuracy.
    
    
    The vast majority of model compression techniques can be divided into four categories: compact architectures~\cite{sandler2018mobilenetv2,liu2021swin}, network pruning~\cite{cheng2024survey,jiang2022sparse}, knowledge distillation~\cite{gou2021knowledge,jiang2023knowledge}, and model quantization~\cite{nagel2021white,li2021brecq}, in which model quantization draws particular attention for its ability to not only reduce model size but also accelerate inference by limiting bit precision. A line of works~\cite{jacob2018quantization,choi2018pact} adopt quantization during model training, which is called quantization-aware training (QAT). While QAT methods achieve high performance, they necessitate end-to-end retraining on the entire training set, resulting in substantial time and computational consumption. To circumvent the heavy retraining, post-training quantization (PTQ) has garnered increasing interest in recent times. PTQ-based methods~\cite{li2021brecq,yuan2022ptq4vit,nagel2020up} perform quantization on pre-trained models using a small-scale calibration set, thereby eliminating the need for end-to-end retraining. 

    \begin{figure}[t]
        \centering
        \includegraphics[width=0.6\columnwidth]{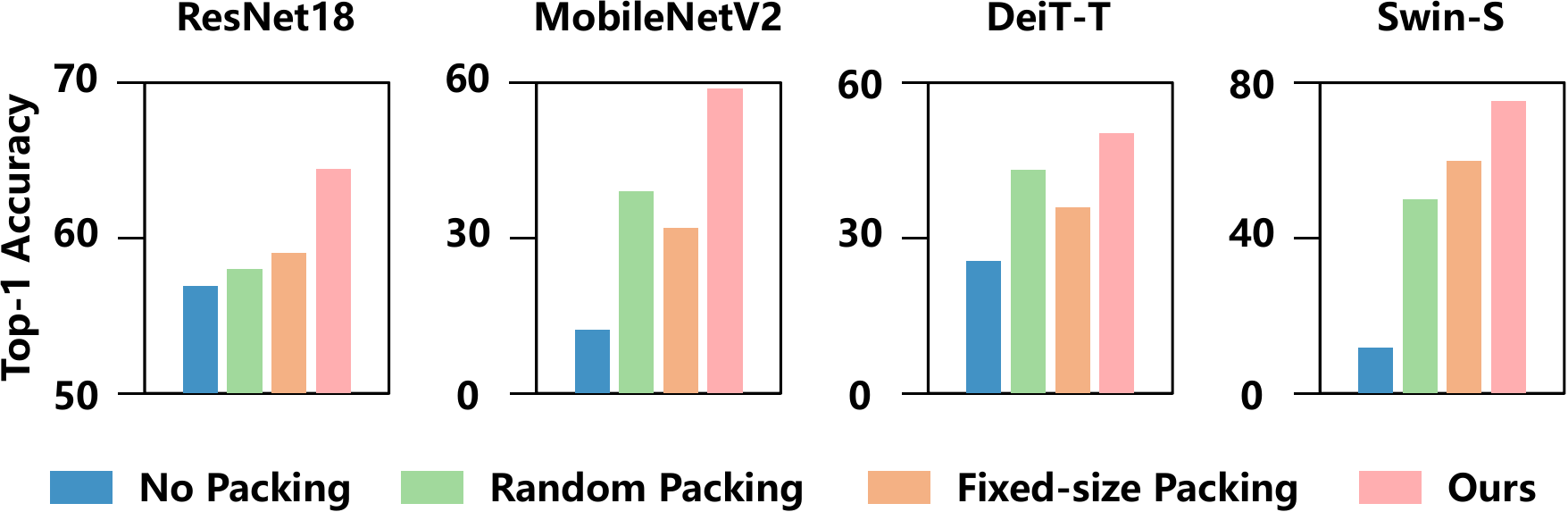}
        \caption{Quantization results of different reconstruction strategies on ImageNet with W3/A3 setting. ``No Packing'' means employing block-wise reconstruction, ``Random Packing'' means randomly assigning blocks into packs, and ``Fixed-size Packing'' means assigning blocks into packs with equal size.}
        \label{fig:motivation}
    \end{figure}

   Extensive PTQ-based studies ~\cite{li2021brecq,wei2022qdrop,nagel2020up} have focused on the reconstruction procedure, which seeks to align the outputs between a quantized model and its full-precision counterpart, to recover accuracy after quantization. However, due to the limited availability of calibration samples, traditional network-wise reconstruction often suffers from over-fitting, leading to subpar test accuracy. To overcome this limitation, recent studies ~\cite{li2021brecq,wei2022qdrop} have proposed employing reconstruction in a block-wise manner, which has been shown to yield improved performance. Unfortunately, as indicated in Fig.~\ref{fig:motivation}, these methods still exhibit limited performance in the case of ultra-low bit quantization, as they frequently underestimate the significance of cross-block dependency, resulting in inaccurate quantization parameters. Therefore, it is essential to revisit the reconstruction granularity of quantization, seeking an optimal balance between capturing cross-block dependencies and maintaining optimization efficacy.
    

    To tackle these challenges, this paper introduces Pack-PTQ, a novel post-training quantization method designed to effectively quantize neural networks, even in low-bit scenarios, as illustrated in Fig.~\ref{fig:motivation}. First, we propose a Hessian-guided adaptive packing mechanism, which adaptively clusters blocks into distinct packs, serving as the base unit for reconstruction. Specifically, this mechanism starts by calculating a Hessian-guided importance score to assess the impact of each block on its former modules, and then clusters blocks into packs with the guidance of these scores. Second, we recognize that different packs exhibit varied sensitivities to quantization, and thus develop a pack-based mixed-precision quantization strategy. This strategy assigns different bit-widths to different packs, demonstrating improved performance compared to unified precision quantization.
    

    The main contributions of this work are summarized in the following four aspects:
    
    \begin{itemize}
        \item We propose Pack-PTQ, a novel solution that combines a Hessian-guided adaptive packing mechanism with a pack-based mixed-precision quantization strategy, enabling high-performance post-training quantization of neural networks, particularly in low-bit cases.
        
        \item We devise a Hessian-guided adaptive packing mechanism to divide a neural network into non-overlapping packs, which provides a new reconstruction granularity, effectively capturing cross-block dependency.
        
        \item With the pack configuration, we develop a pack-based mixed-precision strategy that takes into account the distinct sensitivities of each pack to quantization, thereby enabling a more efficient allocation of bits.
        
        \item Through a comprehensive experimental evaluation on 2D image and 3D point cloud classification tasks, leveraging diverse neural network architectures, we demonstrate the consistent superiority of Pack-PTQ over several state-of-the-art PTQ methods, showcasing its robustness and effectiveness in real-world applications.
    \end{itemize}

\section{Related Works}

\subsection{Post-training Quantization}
   Model quantization is a technique aiming at accelerating inference and reducing memory footprint by limiting the bit-widths of network parameters~\cite{nagel2021white}. A traditional approach to mitigating the accompanying accuracy loss is to perform end-to-end retraining of the model, a method known as quantization-aware training (QAT)~\cite{jacob2018quantization,esser2019learned,nagel2022overcoming}. While QAT-based methods achieve high accuracy, they raise significant concerns regarding privacy and incur substantial time overhead due to the extensive retraining process. 
   
   To address the challenges above, post-training quantization (PTQ) has emerged as a preferred method for generating quantized models without requiring extensive retraining, garnering significant attention in the realm of model compression. A notable advancement in PTQ is BRECQ~\cite{li2021brecq}, which effectively addresses overfitting through block-wise reconstruction. By partially reproducing the Hessian matrix of the original network to serve as a regularizer, BRECQ achieves improved performance in PTQ, making it a promising approach for efficient model compression. Subsequent research has continued to advance the field of PTQ. Q-Drop~\cite{wei2023qdroprandomlydroppingquantization} enhances model robustness by integrating quantized and original data, while NoisyQuant~\cite{liu2023noisyquant} introduces random noise to reduce reliance on fully quantized data, thereby preserving the connectivity of the original model. More recently, SPQ~\cite{padeiro2024lightweight} has made significant contributions to this field by leveraging dataset interconnections to improve quantization, particularly in scenarios with limited calibration data. However, these methods have primarily focused on local data connections, potentially overlooking global inter-dependencies within the network. To address this gap, PD-Quant~\cite{liu2023pd} has introduced global monitoring mechanisms that enhance the accuracy of compressed models. Despite these advancements, challenges persist in fully capturing cross-block relationships, highlighting the ongoing need for improved optimization strategies in PTQ.

\subsection{Mixed-precision Quantization}
    It is commonly acknowledged that quantizing neural networks to 8-bit formats has achieved remarkable results with minimal accuracy degradation~\cite{nagel2019data,yuan2022ptq4vit}. However, quantizing to lower bit-widths (i.e., 3-bit) still remains challenging because of significant accuracy loss. To overcome this hurdle, mixed-precision quantization has emerged as a promising solution, which involves assigning different bit-widths to various layers or components of the models~\cite{chen2021towards}. This approach has shown great potential in balancing accuracy and computational efficiency.
    
    One notable approach in this domain is HAWQ~\cite{dong2019hawq}, which introduces an automated method to determine optimal mixed-precision settings based on layer sensitivity derived from the Hessian spectrum. Building upon this, HAWQ-V2 enhances this method by using the average Hessian trace as a more accurate sensitivity metric compared to the top eigenvalue~\cite{dong2020hawq}. Additionally, HAWQ-V2 employs an automated Pareto frontier method for layer-wise bit precision selection and extends Hessian-based analysis to mixed-precision activation quantization.
    A recent study further contributes to this field by proposing a novel framework that formulates mixed-precision quantization as a discrete constrained optimization problem~\cite{chen2021towards}. This framework introduces a Multiple-Choice Knapsack Problem (MCKP), which is solved using a greedy search algorithm to optimize bit-widths assignments. This approach provides a systematic way to balance the trade-off between model accuracy and computational efficiency. Moreover, APTQ takes into account the nonlinear effects of attention outputs on the entire model, providing a more nuanced approach to quantization~\cite{guan2024aptq}. By considering the impact of attention mechanisms, APTQ offers a refined method for determining the optimal bit-widths.

    \begin{figure*}[t]
        \centering
        \includegraphics[width=\textwidth]{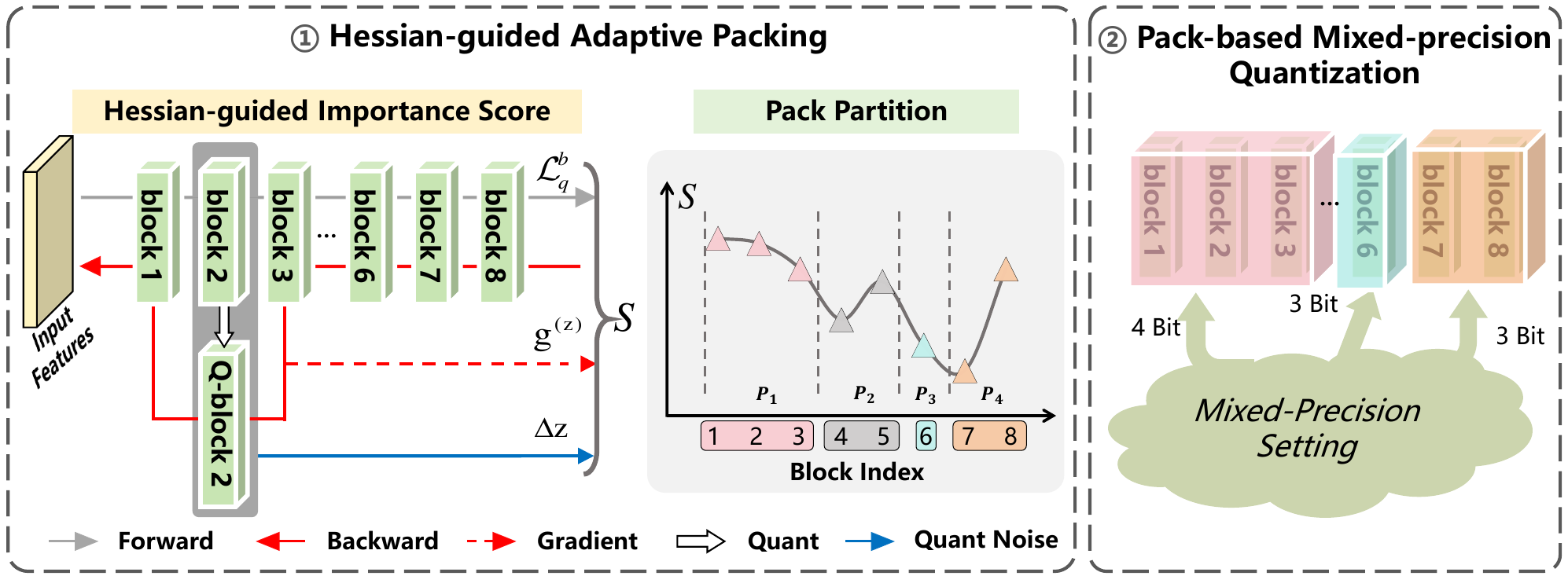}
        \caption{Overview of our proposed Pack-PTQ  method. We begin by computing individual block scores, which take into account the model loss, block output gradients, and block local loss. Subsequently, we employ our novel packing mechanism to cluster these blocks into packs. Finally, we assign diverse bit-widths to each pack, thereby achieving optimal pack reconstruction and facilitating efficient post-training quantization of neural networks.}
        \label{fig:overview}
    \end{figure*}

\section{Preliminaries}

\subsection{Quantization}
    The uniform quantizer is widely adopted due to its superior hardware compatibility, which converts input full-precision numbers to several fixed-points:
    \begin{equation}
       x_{q} = \mathrm{clamp}\left(\left\lfloor{\frac{x}{s} + z}\right\rceil, 0, 2^{k} - 1\right),
    \end{equation}
    where $x$ and $x_{q}$ stand for the input full-precision number and its quantized value, respectively. The $s$, $z$, and $k$ respectively denote the quantization scale, zero point, and bit-widths. $\left\lfloor \cdot \right\rceil$ denotes the rounding operation. In this paper, following BRECQ \cite{li2021brecq}, we employ uniform quantizers for both weights and activations.

\subsection{Taylor Expansion}
In line with previous efforts~\cite{li2021brecq,wei2022qdrop}, quantization imposed on weights can be regarded as a specific instance of weight perturbation. Therefore, the loss degradation caused by quantization can be approximated by Taylor series expansions, which is formulated as:
    \begin{align} \label{eq:taylor}
        \mathcal{L}_{q} = \mathbb{E}[\mathcal{L}(\mathbf{w} + \Delta \mathbf{w})] - \mathbb{E}[\mathcal{L}(\mathbf{w})]
        \approx \Delta \mathbf{w}^\top {\mathbf{g}}^{(\mathbf{w})} + \frac{1}{2} \Delta \mathbf{w}^\top {\mathbf{H}}^{(\mathbf{w})} \Delta \mathbf{w},
    \end{align}
where $\mathcal{L}$ is task-specific loss, $\Delta \mathbf{w}$ represents the noise introduced by quantization, ${\mathbf{g}}^{(\mathbf{w})}$ is the first-order derivative of the loss function. ${\mathbf{H}}^{(\mathbf{w})}$ is the second-order derivative of the loss function, known as the Hessian matrix.

\section{Method}




\subsection{Hessian-guided Adaptive Packing Mechanism} \label{sec:hessian}

    
  Block reconstruction strategy resorts to minimizing the quantization errors by aligning the outputs of the full-precision blocks and their quantized counterparts. This paradigm approximates the Hessian matrix with a block-diagonal matrix, in which off-block entries are set to zero. Despite its efficiency, such a skewed estimation of Hessian underestimates the hidden relationships across different blocks, ultimately leading to inferior reconstruction results. To illustrate this, we modify the reconstruction unit from a single block to a stack of blocks (denoted as packs), which allows us to capture the inter-block dependencies within a pack. As shown in Fig.~\ref{fig:motivation}, such a simple modification largely outperforms the widely used block-wise paradigm, which indicates the significance of exploring the cross-block relationships. In addition, the final performance differs significantly with various packing strategies (i.e., random packing and fixed-size packing), which motivates us to automate the packing configurations. 
  
    
    

    To achieve this, we develop a block importance metric that ranks the blocks and subsequently packs consecutive blocks based on this metric. The cornerstone of our approach is to assess the impact of each block's output on the preceding blocks. In this paper, we propose utilizing the Hessian matrix to form this metric, which inherently captures the second-order information of the proceeding modules. Specifically, we first rewrite Eq.~(\ref{eq:taylor}) to adapt to a single block's output: 
    \begin{align} \label{eq:lqb}
        \mathcal{L}_{q}^{b} = \mathbb{E}[\mathcal{L}(\mathbf{z} + \Delta \mathbf{z})] - \mathbb{E}[\mathcal{L}(\mathbf{z})] 
        \approx \Delta \mathbf{z}^\top {\mathbf{g}}^{(\mathbf{z})} + \frac{1}{2} \Delta \mathbf{z}^\top {\mathbf{H}}^{(\mathbf{z})} \Delta \mathbf{z},
    \end{align} 
    where \(\mathbf{z}\) and $\Delta \mathbf{z}$ represent the output of a block and the change of \(\mathbf{z}\) incurred by quantization, respectively. To fully consider the contributions of all elements in the Hessian matrix, we adopt the average value of the Hessian matrix as the block importance metric. However, directly computing \({\mathbf{H}}^{(\mathbf{z})}\) is infeasible due to the computational complexity of neural networks. Alternatively, we approximate the mean of the Hessian matrix, as detailed below.
    
    \begin{theorem}
        Assume that \( \Delta \mathbf{z} \) is a random vector where each component is i.i.d. as \( \mathcal{N}(0, \sigma^2) \). Then the expectation \(\mu_{{\mathbf{H}}^{(\mathbf{z})}}\) of all the elements in the Hessian matrix \( {\mathbf{H}}^{(\mathbf{z})} \) can be approximated as:
        \begin{equation}
            \mu_{{\mathbf{H}}^{(\mathbf{z})}} \approx \frac{\mathbb{E} [\Delta \mathbf{z}^\top {\mathbf{H}}^{(\mathbf{z})} \Delta \mathbf{z}]}{\mathbb{E} [\Delta \mathbf{z}^\top \Delta \mathbf{z}]}.
        \end{equation}
    \end{theorem}
    
    \begin{proof}
        Utilizing the knowledge of expectations of quadratic forms and the properties of the covariance matrix of random vectors, we have:
        \begin{align} \label{eq:zhz}
            \mathbb{E}[\Delta \mathbf{z}^\top {\mathbf{H}}^{(\mathbf{z})} \Delta \mathbf{z}]
            = \text{tr} ({\mathbf{H}}^{(\mathbf{z})} \mathbb{E} [\Delta \mathbf{z} \Delta \mathbf{z}^\top]) 
            = \text{tr}({\mathbf{H}}^{(\mathbf{z})} \sigma^2 I) 
            \approx \mu_{{\mathbf{H}}^{(\mathbf{z})}} \sigma^2 n.
        \end{align}
        At the same time, the  \( \mathbb{E}[\Delta \mathbf{z}^\top \Delta \mathbf{z}] \) can be concluded as:
        \begin{equation} \label{eq:zz}
            \mathbb{E} [\Delta \mathbf{z}^\top \Delta \mathbf{z}] = \text{tr} (\mathbb{E} [\Delta \mathbf{z} \Delta \mathbf{z}^\top]) = \text{tr}(\sigma^2 I) = \sigma^2 n.
        \end{equation}
        Finally, combining Eq.(\ref{eq:zhz}) and Eq.(\ref{eq:zz}), the $\mu_{{\mathbf{H}}^{(\mathbf{z})}}$ can be obtained by:
        \begin{equation} \label{eq:yy}
            \mu_{{\mathbf{H}}^{(\mathbf{z})}} = \frac{\mu_{{\mathbf{H}}^{(\mathbf{z})}} \sigma^2 n}{\sigma^2 n} \approx \frac{\mathbb{E} [\Delta \mathbf{z}^\top {\mathbf{H}}^{(\mathbf{z})} \Delta \mathbf{z}]}{\mathbb{E} [\Delta \mathbf{z}^\top \Delta \mathbf{z}]}.
        \end{equation}
        \end{proof}
    According to Eq.(\ref{eq:lqb}) and Eq.(\ref{eq:yy}), we obtain the importance score for each block as:
    \begin{equation}
        S = \mu_{{\mathbf{H}}^{(\mathbf{z})}}
        \approx \frac{\mathbb{E} [2 \times (\mathcal{L}_{q}^{b} - \Delta \mathbf{z}^\top {\mathbf{g}}^{(\mathbf{z})})]}{\mathbb{E} [\Delta \mathbf{z}^\top \Delta \mathbf{z}]}.
    \end{equation}
    Notably, this formulation simplifies the computation of the Hessian-based block importance metric to achieve efficiency. A low value of \( S \) indicates that perturbations to the corresponding block have minimal impact on subsequent blocks and even the final output. 

    
    After that, given a set of blocks \( \mathcal{B} = \{ B_1, B_2, \ldots, B_n \} \) with their corresponding importance scores \( \mathcal{S} = \{ S_1, S_2, \ldots, S_n \} \), we aim to assign different blocks to distinct packs. We initiate this process by setting two indicators, $t_s=1$ and $t_e=n$, which define the search scope for the current pack. Then, we identify the block with the lowest importance score within this scope, whose index is denoted as $t_{min} \in [t_s, t_e]$. Based on the above requirements, a new pack is formed by grouping the blocks from $t_{min}$ to $t_e$, which is formally defined as:
    
    \begin{equation}
        P = \left\{ \bigcup_{t=t_{\text{min}}}^{t_e} B_t \;\middle|\; t_{\text{min}} = \arg\min_{t_s \leq t \leq t_e} S[t] \right\}
    \end{equation}
    where \( t \) is the index of a block, and $\mathcal{S}[t]$ denotes the $t$-th entry of $\mathcal{S}$. Once a pack is formed, the $t_e$ is updated to $t_{min}-1$, and the process is repeated until each block is meticulously assigned to a unique pack. This approach partitions the network into non-overlapping packs, where the first block in each pack is characterized by a lower importance score, reflecting its limited influence on the former modules. The blocks within the same pack are tightly coupled, implying they are highly interdependent. By jointly optimizing these blocks, we can derive more accurate quantization parameters, as the correlated behaviors within each pack are effectively captured. We also explore diverse packing partition strategies, which are detailed in Sec.~\ref{exp_abla}.

\subsection{Pack-based Mixed-precision Quantization}
As previous studies \cite{dong2019hawq,chen2021towards} have shown, different layers exhibit varying sensitivities to quantization, implying that a unified bit-widths assignment across all layers is suboptimal. Specifically, some blocks are relatively insensitive to low-bit quantization and can tolerate aggressive compression, while others require higher precision to maintain accuracy. Motivated by this, we propose a novel approach that assigns different bit-widths to different packs, striking a balance between performance and memory footprint. To be specific, assume that the constraint of memory footprint is $C$ and the pre-defined candidate bit-widths set is $\mathcal{K} = \{k_1, k_2, ..., k_n \}$, and our goal is to determine the optimal bit-widths configuration that satisfies the memory constraint while maximizing model accuracy. To achieve this, we formulate the problem as an optimization objective, which can be mathematically expressed as:

\begin{equation} \label{eq:mix-q}
    \max_{\{b_1, b_2, ..., b_M\}} \sum_{j=1}^{M} b_j \cdot \Omega_j, \quad \text{s.t.} \quad \sum_{j=1}^{M} b_j \cdot p_j \leq C,
\end{equation} 
where \( M \) and \( p_j \) denote the number of packs and parameters in the \( j \)-th pack, respectively. The $b_j \in \mathcal{K}$ is the bit-width for the $j$-th pack, and the $\Omega_j$ is the mean sensitivity for the $j$-th pack, which is defined as: 
\begin{equation}
    \Omega_j = \frac{1}{n_j} \sum_{i=1}^{n_j} (S_{j}[i] \cdot \mathcal{L}^q_{j}[i]),
\end{equation}
where \(n_j\) denotes the number of blocks in \( j\)-th pack, \( S_{j}[i] \) represents the $i$-th importance score of the block in the $j$-th pack, and \( \mathcal{L}^q_{j}[i] \) denotes the quantization loss of the $i$-th block in the $j$-th pack. By solving Eq.~(\ref{eq:mix-q}), packs with higher sensitivity are allocated higher bit-widths, thereby preserving their accuracy, while packs with lower sensitivity are subjected to more aggressive quantization, thereby reducing memory footprint. Such a way enables flexible and efficient mixed-precision quantization, improving network performance while adhering to memory limitations.





\subsection{Optimization}
    To recover the accuracy of quantized models, we leverage the packs obtained in Sec.~\ref{sec:hessian} as the base unit for reconstruction. Specifically, we define a pack-wise reconstruction loss, which is formulated as:
    \begin{equation}
        \mathcal{L}_{rec} = \mathbb{E}[||P^{(l)}_q(x^{(l)}) - P^{(l)}(x^{(l)})||_{\mathrm F}],
    \end{equation}
    where $x^{(l)}$ is the input for the $l$-th pack and $||\cdot||_{\mathrm F}$ represents the Frobenius norm. $P^{(l)}_q(x^{(l)})$ and $P^{(l)}(x^{(l)})$ denote the outputs of the $l$-th quantized pack and full-precision pack, respectively. Notably, the reconstruction granularity can differ among packs since the packs have varied sizes, enabling the model to adaptively learn cross-block dependencies and capture nuanced relationships between blocks. Thus, such a manner facilitates the learning of more accurate quantization parameters, which is essential for preserving the model's accuracy during the quantization process. 

\section{Experiments} \label{sec:exp}

\subsection{Experimental Setups}


    \subsubsection{Datasets and Competing Methods} 

    
    We conduct comprehensive experiments on image classification and point cloud classification tasks. For image classification, Pack-PTQ is evaluated on the ImageNet dataset~\cite{krizhevsky2012imagenet}, which consists of 1,000 categories and 50,000 test samples. 
    {We use both} Convolutional Neural Networks (CNNs) and Vision Transformers (ViTs) as network architectures to validate the effectiveness of our method.
    To be specific, in the case of CNNs, we compare Pack-PTQ with several state-of-the-art methods, including BRECQ~\cite{li2021brecq}, RAPQ~\cite{yao2022rapq}, MRECG~\cite{ma2023solving}, and Genie~\cite{jeon2023genie}. For ViTs, the evaluation is conducted against BRECQ~\cite{li2021brecq}, QDrop~\cite{wei2022qdrop}, PTQ4ViT~\cite{yuan2022ptq4vit}, PD-Quant~\cite{liu2023pd}, and RepQ-ViT~\cite{li2023repq}. As for the point cloud classification, we evaluate our Pack-PTQ on ModelNet40 \cite{wu20153d} which consists of 40 categories with a test set of 2,468 samples. To the best of our knowledge, there are no previous works dedicated to post-training quantization on ModelNet40. Therefore, we reimplement the MinMax~\cite{nagel2020up} and BRECQ~\cite{li2021brecq} with their official codes as our competing methods.
    
    \subsubsection{Implementation Details} 

    {In our experiments, we utilize pre-trained full-precision transformer models from the Timm library\footnote{\url{https://github.com/rwightman/pytorch-image-models}} and CNN models from BRECQ~\cite{li2021brecq}. 
    We apply uniform quantization to all model weights and activations. 
    We adopt the Adam optimizer with an initial learning rate of 4e-5, no weight decay, and a cosine decay schedule to adjust the learning rate.}
    All experiments are conducted with a batch size of 32 across 2,000 training iterations, following the protocol outlined in prior work~\cite{zhong2023s}. For the image classification task, Pack-PTQ is evaluated on two architecture families: CNNs and ViTs. For CNNs, the evaluation spans various architectures, including ResNet~\cite{he2016deep} (standard convolution), MobileNetV2~\cite{sandler2018mobilenetv2} (depthwise separable convolution), RegNet~\cite{radosavovic2020designing} (group convolution), and MNasNet~\cite{tan2019mnasnet}. For ViTs, it includes ViT~\cite{dosovitskiy2020image}, DeiT~\cite{touvron2021training}, and Swin Transformer~\cite{liu2021swin}. For the point cloud classification task, evaluations are conducted exclusively on PointNet \cite{qi2017pointnet}. All experiments are performed using a single NVIDIA TITAN RTX GPU. The source code is released in our supplementary materials.

\subsection{Results on Image Classification}


    {Table \ref{tab:booktabs_cnns} shows the test accuracy of various methods on ImageNet for different CNN models and bit-widths. As shown in Table \ref{tab:booktabs_cnns}, our Pack-PTQ consistently outperforms existing methods in both W3/A3 and W4/A4 quantization settings. Notably, without mixed-precision (w/o MP), Pack-PTQ surpasses Genie, the second-best performer, by 8.15\% on MobileNetV2 under W3/A3 quantization. When using mixed-precision, the improvement increases to 12.86\%, demonstrating the effectiveness of bit-width allocation.
    Furthermore, our method significantly outperforms BRECQ, which achieves an accuracy of only 12.27\% under the same setting. The performance gap between our method and BRECQ is particularly pronounced on MobileNetV2, primarily due to the absence of cross-block dependency in BRECQ. Although our method ranks second on some architectures under W2/A4 quantization, the performance differences are minor. For instance, our Pack-PTQ achieves 61.29\% accuracy on ResNet18, only 0.21\% behind BRECQ, the best performer. However, our method demonstrates superior generalization performance across all architectures. Moreover, our results show minimal performance degradation when using lower bit-widths, highlighting the robustness of our Pack-PTQ.}

     \begin{table*}[t!]
        \centering
        \caption{Quantization results for CNNs on the ImageNet dataset. The top-1 accuracy (\%) is reported. ``Bit. (W/A)'' represents the bit width for weights and activations. \textbf{Bold} font indicates the highest value. \text{\underline{Underlined}} data represents the second highest value.}
            \begin{tabular}{ lcrrrrrr }
                \toprule
                \textbf{Method}    & \textbf{Bit. (W/A)}  & \textbf{ResNet18} & \textbf{ResNet50} & \textbf{MNV2} & \textbf{Reg600M} & \textbf{Reg3.2G} & \textbf{MNasx2} \\
                \midrule
                \text{Full-Precision}  & \text{32/32}          & \text{71.08}   & \text{77.00}   & \text{72.49}    & \text{73.71}    & \text{78.36}  & \text{76.68} \\
                \midrule
                \text{BRECQ \cite{li2021brecq}}            & \text{2/4}          & \textbf{61.50} & \text{\underline{62.79}} & \text{21.18} & \text{56.49} & \text{\underline{59.51}} & \text{49.11} \\
                \text{RAPQ \cite{yao2022rapq}}        & \text{2/4}          & \text{61.18} & \text{59.70} & \text{34.47} & \text{54.80} & \text{57.47} & \text{-} \\
                \text{MRECG \cite{ma2023solving}}           & \text{2/4}          & \text{59.19} & \text{57.96} & \text{47.08} & \text{26.26} & \text{-} & \text{-} \\
                \text{Genie \cite{jeon2023genie}}           & \text{2/4}          & \text{61.27} & \text{57.27} & \text{\underline{50.44}} & \text{\underline{57.77}} & \text{56.74} & \textbf{62.24} \\
                \rowcolor{gray!20} \text{Pack-PTQ (Ours) w/o MP}           & \text{2/4}          & \text{\underline{61.29}} & \textbf{63.47} & \textbf{55.86} & \textbf{59.88} & \textbf{63.14} & \text{\underline{59.17}} \\
                \midrule
                \text{BRECQ \cite{li2021brecq}}            & \text{3/3}          & \text{56.89} & \text{61.84} & \text{12.27} & \text{47.85} & \text{51.14} & \text{41.00} \\
                \text{RAPQ \cite{yao2022rapq}}        & \text{3/3}          & \text{55.44} & \text{58.71} & \text{7.33} & \text{44.94} & \text{51.04} & \text{-} \\
                \text{MRECG \cite{ma2023solving}}           & \text{3/3}          & \text{60.01} & \text{66.82} & \text{48.87} & \text{57.75} & \text{-} & \text{-} \\
                \text{Genie \cite{jeon2023genie}}           & \text{3/3}          & \text{61.10} & \text{67.96} & \text{50.68} & \text{59.53} & \text{67.11} & \text{59.59} \\
                \text{TexQ \cite{chen2024texq}}           & \text{3/3}          & \text{50.28} & \text{25.27} & \text{32.80} & \text{-} & \text{-} & \text{-} \\
                \rowcolor{gray!20} \text{Pack-PTQ (Ours) w/o MP}           & \text{3/3}          & \text{\underline{64.46}} & \text{\underline{70.34}} & \text{\underline{58.83}} & \text{\underline{65.12}} & \text{\underline{70.81}} & \text{\underline{62.71}} \\
                \rowcolor{gray!20} \text{Pack-PTQ (Ours) with MP} & \text{3/3} & \textbf{66.73} & \textbf{72.14} & \textbf{63.54} & \textbf{68.67} & \textbf{72.88} & \textbf{66.76} \\
                \midrule
                \text{BRECQ \cite{li2021brecq}}            & \text{4/4}          & \text{68.20} & \text{73.20} & \text{61.93} & \text{68.91} & \text{74.09} & \text{69.57} \\
                \text{RAPQ \cite{yao2022rapq}}        & \text{4/4}          & \text{68.11} & \text{73.08} & \text{60.50} & \text{68.25} & \text{73.64} & \text{-} \\
                \text{MRECG \cite{ma2023solving}}           & \text{4/4}          & \text{67.69} & \text{73.96} & \text{66.55} & \text{69.22} & \text{-} & \text{-} \\
                \text{Genie \cite{jeon2023genie}}           & \text{4/4}          & \text{67.80} & \text{74.07} & \text{66.85} & \text{69.50} & \text{75.23} & \text{69.83} \\
                \text{TexQ \cite{chen2024texq}}           & \text{4/4}          & \text{67.73} & \text{70.72} & \text{67.07} & \text{-} & \text{-} & \text{-} \\
                \rowcolor{gray!20} \text{Pack-PTQ (Ours) w/o MP}           & \text{4/4}          & \text{\underline{68.74}} & \text{\underline{74.74}} & \text{\underline{68.58}} & \text{\underline{70.96}} & \text{\underline{76.83}} & \text{\underline{72.51}} \\
                \rowcolor{gray!20} \text{Pack-PTQ (Ours) with MP}           & \text{4/4}          & \textbf{69.86} & \textbf{75.51} & \textbf{70.41} & \textbf{72.31} & \textbf{77.70} & \textbf{74.27} \\
                \bottomrule
            \end{tabular}
        \label{tab:booktabs_cnns}
    \end{table*}

    \begin{table*}[t!]
        \centering
        \caption{Quantization results for ViTs on the ImageNet dataset. The top-1 accuracy (\%) is reported. ``Bit. (W/A)'' represents the bit width for weights and activations. \textbf{Bold} font indicates the highest value. \text{\underline{Underlined}} data represents the second highest value.}
            \begin{tabular}{lccrrrrrrrrrr}
                \toprule
                \textbf{Method}    & \textbf{Bit. (W/A)}  & \textbf{ViT-S} & \textbf{ViT-B} & \textbf{DeiT-T} & \textbf{DeiT-S} & \textbf{DeiT-B} 
                & \textbf{Swin-S}    & \textbf{Swin-B} \\
                \midrule
                \text{Full-Precision}& \text{32/32}          & \text{81.39}   & \text{84.54}   & \text{72.21}    & \text{79.85}    & \text{81.80} 
                & \text{83.23}         & \text{85.27} \\
                \midrule
                \text{BRECQ~\cite{li2021brecq}}        & \text{3/3}            & \text{0.42}   & \text{0.59}   & \text{25.52}    & \text{14.63}    & \text{46.29} 
                & \text{11.67}         & \text{1.70} \\
                \text{QDrop~\cite{wei2022qdrop}}         & \text{3/3}            & \text{4.44}   & \text{8.00}   & \text{30.73}    & \text{22.67}    & \text{24.37} 
                & \text{60.89}         & \text{54.76} \\
                \text{PTQ4ViT~\cite{yuan2022ptq4vit}}       & \text{3/3}            & \text{0.01}   & \text{0.01}   & \text{0.04}    & \text{0.01}    & \text{0.27} 
                & \text{0.35}         & \text{0.29} \\
                \text{PD-Quant~\cite{liu2023pd}}     & \text{3/3}            & \text{1.77}   & \text{13.09}   & \text{39.97}    & \text{29.33}    & \text{0.94} 
                & \text{69.67}         & \text{64.32} \\
                \text{RepQ-ViT~\cite{li2023repq}}      & \text{3/3}            & \text{0.43}   & \text{0.14}   & \text{0.97}    & \text{4.37}    & \text{4.84} 
                & \text{8.84}         & \text{1.34} \\
                \text{Adalog~\cite{wu2025adalog}}      & \text{3/3}            & \text{10.32}   & \text{30.10}   & \text{23.06}    & \text{19.61}    & \text{53.33} 
                & \text{54.66}         & \text{62.38} \\
                \rowcolor{gray!20} \text{Pack-PTQ (Ours) w/o MP}   & \text{3/3}            & \text{\underline{47.68}}   & \text{\underline{61.65}}   & \text{\underline{50.31}}    & \text{\underline{53.04}}    & \text{\underline{70.13}} 
                & \text{\underline{75.42}}         & \text{\underline{76.04}} \\
                \rowcolor{gray!20}  \text{Pack-PTQ (Ours) with MP}    & \text{3/3}           & \textbf{55.56}    & \textbf{64.83} 
                & \textbf{58.19}   & \textbf{58.38}   & \textbf{74.41}    & \textbf{78.92}         & \textbf{80.49} \\
                \midrule
                \text{BRECQ~\cite{li2021brecq}} & \text{4/4} & \text{12.36} & \text{9.68} & \text{55.63} & \text{63.73} & \text{72.31} & \text{72.74} & \text{58.24} \\
                \text{QDrop~\cite{wei2022qdrop}} & \text{4/4} & \text{21.24} & \text{47.30} & \text{61.93} & \text{68.27} & \text{72.60} & \text{79.58} & \text{80.93} \\
                \text{PTQ4ViT~\cite{yuan2022ptq4vit}} & \text{4/4} & \text{42.57} & \text{30.69} & \text{36.96} & \text{34.08} & \text{64.39} & \text{76.09} & \text{74.02} \\
                \text{APQ-ViT~\cite{ding2022towards}} & \text{4/4} & \text{47.95} & \text{41.41} & \text{47.94} & \text{43.55} & \text{67.48} & \text{77.15} & \text{76.48} \\
                \text{PD-Quant~\cite{liu2023pd}} & \text{4/4} & \text{1.51} & \text{32.45} & \text{62.46} & \text{71.21} & \text{73.76} & \text{79.87} & \text{81.12} \\
                \text{RepQ-ViT~\cite{li2023repq}} & \text{4/4} & \text{65.05} & \text{68.48} & \text{57.43} & \text{69.03} & \text{75.61} & \text{79.45} & \text{78.32} \\
                \text{Adalog~\cite{wu2025adalog}}      & \text{4/4}            & \textbf{72.13}   & \textbf{79.16}   & \text{62.14}    & \text{\underline{71.88}}    & \text{77.47} 
                & \text{79.27}         & \text{80.92} \\
                 \rowcolor{gray!20} \text{Pack-PTQ (Ours) w/o MP}   & \text{4/4}            & \text{61.03}   & \text{75.71}   & \text{\underline{64.53}}    & \text{71.57}    & \text{\underline{78.45}} 
                & \text{\underline{80.82}}         & \text{\underline{82.34}} \\
               \rowcolor{gray!20}  \text{Pack-PTQ (Ours) with MP}   & \text{4/4}           & \text{\underline{66.83}}    & \text{\underline{77.89}} 
                & \textbf{68.24}   & \textbf{75.33}   & \textbf{79.41}    & \textbf{82.18}         & \textbf{83.75} \\
                \bottomrule
            \end{tabular}
        \label{tab:booktabs_vits}
    \end{table*}

    {We further evaluate our Pack-PTQ on several representative ViT models, as shown in Table \ref{tab:booktabs_vits}. Overall, our Pack-PTQ achieves state-of-the-art performance in almost all cases, with the smallest gap with the full-precision model. Notably, in the W3/A3 setting, most competing methods exhibit significantly degraded performance compared to full-precision models. For instance, RepQ-ViT and PTQ4ViT achieve only 0.14\% and 0.01\% accuracy on ViT-B and DeiT-S, respectively, highlighting the challenges in quantizing ViTs in low-bit cases. In contrast, our method achieves accuracy of 61.65\% and 50.31\% without mixed-precision, and further reaches 64.83\% and 58.19\% with mixed-precision, respectively. 
    Although Adalog outperforms Pack-PTQ in terms of W4/A4 quantization on ViT-S and ViT-B, its complex quantizer for post-Softmax activations limits the efficiency of quantized models. Moreover, Adalog's advantages do not generalize to the W3/A3 setting, where our method significantly outperforms it.
    The results in Table \ref{tab:booktabs_vits} also underscore the importance of cross-block dependencies in maintaining performance during quantization, as evidenced by BRECQ's effectiveness on CNNs but poor performance on ViTs. 
    }
    

\subsection{Results on Point Cloud Classification}

    {We also conduct experiments on the 3D point cloud classification task, as shown in Table \ref{tab:booktabs_modelnet40}. Overall, our method outperforms all other methods in all settings, achieving impressive performance comparable to that of full-precision models. Notably, our method exhibits negligible performance degradation compared to the full-precision model in the W2/A4 and W3/A3 quantization settings without mixed-precision. Specifically, we observe a slight decrease of 0.45\% and 0.17\% for W2/A4, and 0.46\% and 0.13\% for W3/A3 quantization in terms of mAcc and OA, respectively.
    Moreover, in the W3/A3 setting, our method even surpasses the full-precision performance when using mixed precision, and consistently outperforms the full-precision model in the W4/A4 quantization setting regardless of whether mixed precision is used. This suggests that point cloud networks may inherently possess higher precision redundancy, which mitigates the impact of bit-width on the model's generalization ability.
    }

\begin{table*}[ht]
\centering
        \caption{Quantization results on ModelNet40. ``mACC'' is the mean of class-wise accuracy (\%) and ``OA'' is the overall accuracy (\%).}
\begin{tabular}{lc|rr}
            \toprule
            \textbf{Method} & \textbf{Bit. (W/A)} & \textbf{mAcc} & \textbf{OA} \\
            \midrule
            \text{Full-Precision} & \text{32/32} & \text{92.01} & \text{88.54} \\
            \midrule
            \text{MinMax} & \text{2/4} & \text{7.51} & \text{12.85} \\
            \text{BRECQ~\cite{li2021brecq}} & \text{2/4} & \text{51.25} & \text{43.41} \\
            \rowcolor{gray!20} \text{Pack-PTQ (Ours) w/o MP} & \text{2/4} & \textbf{91.56} & \textbf{88.37} \\
            \midrule
            \text{MinMax} & \text{3/3} & \text{78.95} & \text{75.77} \\
            \text{BRECQ~\cite{li2021brecq}} & \text{3/3} & \text{67.48} & \text{61.18} \\
            \rowcolor{gray!20} \text{Pack-PTQ (Ours) w/o MP} & \text{3/3} & \text{91.55} & \text{88.41} \\
            \rowcolor{gray!20} \text{Pack-PTQ (Ours) with MP} & \text{3/3} & \textbf{92.28} & \textbf{89.53} \\
            \midrule
            \text{MinMax} & \text{4/4} & \text{91.44} & \textbf{88.98} \\
            \text{BRECQ~\cite{li2021brecq}} & \text{4/4} & \text{87.25} & \text{82.70} \\
            \rowcolor{gray!20} \text{Pack-PTQ (Ours) w/o MP} & \text{4/4} & \textbf{92.37} & \text{88.71} \\
            \rowcolor{gray!20} \text{Pack-PTQ (Ours) with MP} & \text{4/4} & \text{92.21} & \text{88.56} \\
            \bottomrule
\end{tabular}
        \label{tab:booktabs_modelnet40}
\end{table*}
    
    
    BRECQ suffers significant performance degradation in the W3/A3 setting, with a notable drop of 11.47\% in mAcc and 14.59\% in OA compared to MinMax. This degradation can be attributed to the fact that block-wise overlooks the cross-block relationships, thus leading to inaccurate quantitation parameters. In stark contrast, our method demonstrates remarkable stability and consistency across all scenarios, showcasing its robustness and effectiveness in preserving model performance even at lower bit-widths.

\subsection{Ablation Studies} \label{exp_abla}

\subsubsection{Effectiveness of Key Components of Pack-PTQ }

    {To validate the effectiveness of block packing and mixed-precision (MP) optimization, we conduct ablation studies on ResNet18 and ViT-S models using the W3/A3 quantization setting, as shown in Table~\ref{tab:booktabs_ablation}. We begin by evaluating the block-wise optimization (without any variants) as the baseline, which yields inferior performance, with ViT-S achieving only 0.42\% accuracy.
    We then examine the impact of applying random packing and mixed-precision settings separately, as shown in the second and fourth rows of the table. 
    These components lead to improvements for both architectures, with ViT-S achieving significant increases of 35.10\% and 39.93\% accuracy, respectively.
    Notably, replacing the random packing strategy with our HAda mechanism further enhances performance, particularly on the ViT architecture, where accuracy improves to 47.68\% (+12.16\%), as evident from the comparison between rows 2 and 3. This demonstrates the effectiveness of our HAda mechanism.
    Furthermore, combining block packing and MP optimization yields even better performance. Among the various configurations, jointly using HAda and MP (ours) proves to be the most advantageous, achieving the highest accuracy of 66.73\% and 55.56\% on ResNet18 and ViT-S, respectively. 
    }

    \begin{figure}[!t]
        \centering
        \includegraphics[width=0.6\linewidth]{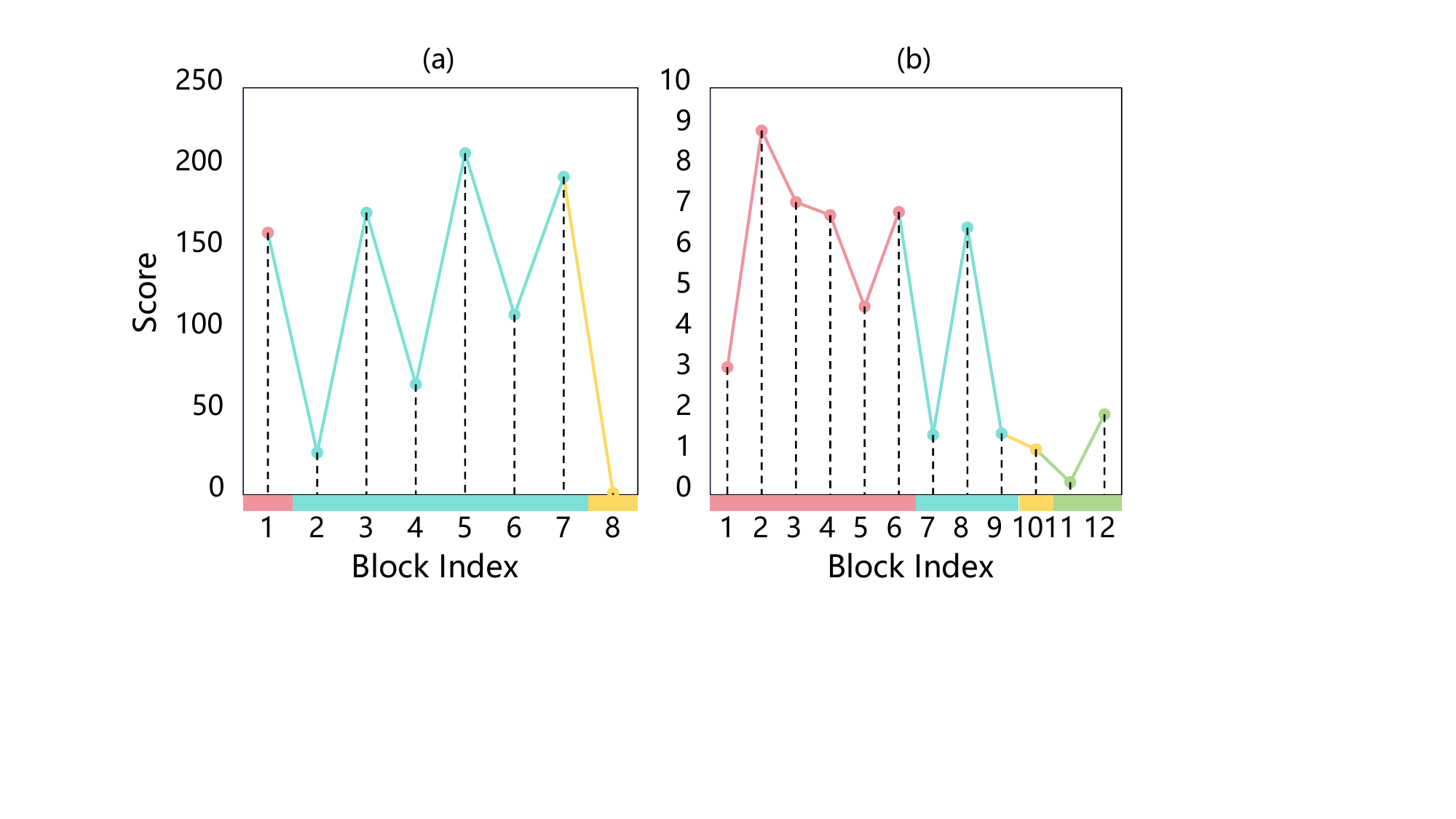} 
        \caption{Visualization of important scores of blocks on (a) ResNet18 and (b) ViT-S. Different colors denote different packs.}
        \label{fig:visualize}
    \end{figure}

 \begin{table*}[!h]
        \centering
        \caption{Top-1 accuracy (\%) of different variants. ``Random'', ``HAda'', and ``MP'' denote the random packing, Hessian-guided adaptive packing, and mixed-precision respectively.}
        \begin{tabular}{ccc|cc}
            \toprule
            \multicolumn{3}{c}{\textbf{Method}} & \multicolumn{2}{c}{\textbf{Accuracy}}       \\
            \cmidrule{1-5} 
            Random    & HAda    & MP   & ResNet18   & ViT-S \\
            \midrule
            $\times$     & $\times$  & $\times$ & 56.89 & 0.42     \\
            $\checkmark$ & $\times$  & $\times$ & 62.81 & 35.52    \\
            $\times$ & $\checkmark$ & $\times$ & 64.46 & 47.68    \\
            $\times$     & $\times$  & $\checkmark$ & 56.82 & 40.35    \\
            $\checkmark$ & $\times$  & $\checkmark$ & 64.53 & 47.86    \\
            $\times$ & $\checkmark$ & $\checkmark$ & \textbf{66.73} & \textbf{55.56}    \\
            \bottomrule
        \end{tabular}
        \label{tab:booktabs_ablation}
    \end{table*}
    

  We visualize the learned packs generated by our Hessian-guided adaptive packing mechanism in Figure~\ref{fig:visualize}. It reveals an interesting phenomenon that the earlier blocks in the network tend to cluster into a single pack. This suggests that these blocks are tightly coupled and should be optimized together, highlighting the importance of our packing strategy in capturing the complex dependencies within the network.

\subsubsection{Training Efficiency Analysis}

As shown in Table~\ref{tab:exe-time}, Pack-PTQ exhibits high efficiency across various models, highlighting its practicality in execution time. Notably, it only requires 0.79 hours on ResNet18 and 1.80 hours on MNV2. For MNasNetx2, ViT-S and DeiT-S, Pack-PTQ remains competitive, with execution times of approximately 2.3 hours, indicating its scalability. Notwithstanding, the results for more complex models (e.g., Swin-S) indicate that additional optimization efforts should be explored in future research to improve efficiency.

\begin{table*}[htpb]
\centering
\caption{Execution time (hours) of Pack-PTQ. }
\renewcommand{\arraystretch}{1.2} 
\setlength{\tabcolsep}{8pt} 
\begin{tabular}{lcclcc}
\toprule
\multicolumn{2}{c}{\textbf{CNN Models}} & & \multicolumn{2}{c}{\textbf{ViT Models}} \\
\cmidrule(lr){1-2} \cmidrule(lr){4-5}
\textbf{Model} & \textbf{Exe. Time} & & \textbf{Model} & \textbf{Exe. Time} \\
\midrule
    ResNet18  & 0.79           & & ViT-S   & 2.31 \\
    MNV2      & 1.80           & & DeiT-S  & 2.27 \\
    MNasx2    & 2.25           & & Swin-S  & 6.82 \\
\bottomrule
\end{tabular}
\label{tab:exe-time}
\end{table*}

\section{Conculsion}
    This paper presents Pack-PTQ, a novel method for post-training quantization of neural networks. Pack-PTQ assigns a Hessian-based importance score to each block, which enables the network to be divided into non-overlapping packs that serve as the base unit for reconstruction. Furthermore, to promote accuracy, Pack-PTQ introduces a pack-based mixed-precision quantization approach that leverages aggressive low-bit quantization for packs with lower sensitivities and high-bit quantization for those with high sensitivities. Extensive experiments on multiple vision tasks, including image classification and point cloud classification with diverse network architectures (e.g., CNNs, ViTs, and PointNet), demonstrate that complex neural networks can be compressed into low-bit versions without significantly compromising performance. In future work, we plan to extend our Pack-PTQ to more complicated vision tasks, and focus on exploring more efficient methods for pack reconstruction, to reduce the computational overhead associated with this process. 



\end{document}